\documentclass[12pt]{amsart}

\usepackage[foot]{amsaddr}
\usepackage[margin = 1 in]{geometry}
\usepackage{graphicx}

\usepackage{booktabs} 

\usepackage{amssymb,amsmath,amscd}

\usepackage{epsfig}
\usepackage{bm}
\usepackage{mathtools}
\usepackage{color}
\usepackage{nicefrac}

\usepackage{lmodern}
\usepackage{multirow}
\usepackage{bbm}
\usepackage{units}
\usepackage{hyperref}
\usepackage{longtable}
\usepackage{color}

\usepackage{csquotes}

\usepackage{tikz}
\tikzset{fontscale/.style = {font=\relsize{11pt}}}
\usetikzlibrary{decorations.markings}

\usepackage{latexsym}

\newtheorem{thm}{Theorem}[section]
\newtheorem{lem}[thm]{Lemma}
\newtheorem{prop}[thm]{Proposition}

\newcommand{\R}{\mathbb{R}}

\newcommand{\hp}{\operatorname{h_p}}
\newcommand{\M}{\operatorname{M}}

\renewcommand{\M}{\mathbf{M}}

\AtBeginDocument{\renewcommand{\d}[1]{d^{#1}}}
\newcommand{\dN}[1]{d_N^{#1}}

\begin{document}


\title{A new class of metrics for learning on real-valued and structured data}







\author{Ruiyu Yang}
\author{Yuxiang Jiang}
\author{Scott Mathews}
\author{Elizabeth A. Housworth}
\author{Matthew W. Hahn}
\author{Predrag Radivojac}
\address{Indiana University, Bloomington, Indiana, U.S.A.}
\email{\{ry2, yuxjiang, scomathe, ehouswor, mwh, predrag\}@indiana.edu}

\keywords{Distance, metric, ontology, machine learning, text mining, high-dimensional data, computational biology}


\maketitle

\begin{abstract}
We propose a new class of metrics on sets, vectors, and functions that can be used in various stages of data mining, including exploratory data analysis, learning, and result interpretation. These new distance functions unify and generalize some of the popular metrics, such as the Jaccard and bag distances on sets, Manhattan distance on vector spaces, and Marczewski-Steinhaus distance on integrable functions. We prove that the new metrics are complete and show useful relationships with $f$-divergences for probability distributions. To further extend our approach to structured objects such as concept hierarchies and ontologies, we introduce information-theoretic metrics on directed acyclic graphs drawn according to a fixed probability distribution. We conduct empirical investigation to demonstrate intuitive interpretation of the new metrics and their effectiveness on real-valued, high-dimensional, and structured data. Extensive comparative evaluation demonstrates that the new metrics outperformed multiple similarity and dissimilarity functions traditionally used in data mining, including the Minkowski family, the fractional $L^p$ family, two $f$-divergences, cosine distance, and two correlation coefficients. Finally, we argue that the new class of metrics is particularly appropriate for rapid processing of high-dimensional and structured data in distance-based learning.
\end{abstract}

\section{Introduction}
The development of domain-specific learning algorithms inevitably requires choices regarding data preprocessing, data representation, training, model selection, or evaluation. One such requirement permeating all of data mining is the selection of similarity or distance functions. In a supervised setting, for example, the nearest neighbor classifiers \cite{Cover1967} and kernel machines \cite{Shawe2004} critically depend on the selection of distance functions. Similarly, the entire classes of clustering techniques rely on the distances that are sensible for a particular domain \cite{Tan2006}. Modern applications further require that distance functions be fast to compute, easy to interpret, and effective on high-dimensional data.

We distinguish between distances and distance metrics; i.e., functions that impose constraints on the general notion of distance \cite{book}. Although restrictions to metrics are not required in data mining \cite{Margareta}, a number of algorithms rely on the existence of metric spaces either explicitly or implicitly. Metric-associated benefits include well-defined point neighborhoods, advanced indexing through metric trees, provable convergence, guarantees for embedding, and intuitive result interpretation. Satisfying metric properties is therefore desirable and generally leads to computational speed-ups and better inference outcomes \cite{Moore2000, Elkan, DBScan, Hamerly2010, Simovici}.

In this work we present a new class of metrics on sets, vectors, and functions that satisfy all of the aforementioned properties. We identify well-known special cases and then show how these distance functions can be adapted to give rise to information-theoretic metric spaces on sets of directed acyclic graphs that are used as class labels in high-cardinality structured-output learning. We prove useful properties of the new metrics and then carry out experiments to assess their suitability in real-life applications. The new metrics exhibited good intuitive behavior and performed favorably against all similarity and dissimilarity functions evaluated in this work, including the Euclidean distance, cosine distance, Pearson's correlation coefficient, Spearman's rank correlation coefficient, and others.

The remainder of this paper is organized as follows. In Section \ref{sec:motivation} we give a motivating example for this work and state the contributions. In Sections \ref{methods}-\ref{connection} we present new metrics and prove useful theoretical properties. In Section \ref{experiments} we carry out empirical evaluation on several types of data. In Section \ref{ontology} we introduce metrics on ontologies and evaluate their performance on problems in computational biology. Sections \ref{experiments}-\ref{ontology} also give performance insights and discuss computational complexity. In Section \ref{relatedwork} we summarize the related work. Finally, in Section \ref{conclusion} we draw conclusions considering both theoretical and empirical findings of our study.

\section{Motivation and contributions}
\label{sec:motivation}

\subsection{A motivating example.} The selection of distance functions and understanding of their behavior is fundamental to data mining. Inference algorithms such as k-nearest neighbor (KNN) classification and K-means clustering rely directly on user-provided distance functions and are among the most popular techniques in the field \cite{Wu2009}. Although both algorithms permit the use of any distance measure, they are generally used with the Euclidean distance.\footnote{K-means algorithm aims to group the data so as to minimize the sum-of-squared errors objective; i.e., the sum of squared Euclidean distances between data points and their respective centroids \cite{Tan2006}.} There is ample evidence, however, that Euclidean distance performs poorly in high-dimensional spaces, leading to a body of theoretical and practical work towards understanding its properties \cite{beyer1999nearest, hinneburg2000nearest, Aggarwal2001b, radovanovic2010hubs} and raising questions about best practices in the field. Other distances; e.g., fractional distances or cosine typically improve the performance in practice. However, these distances are not metrics (e.g., they violate triangle inequality), and so applications using them relinquish theoretical guarantees reserved only for distance metrics. For example, well-known accelerations for K-means clustering only apply to metric spaces \cite{Elkan, Hamerly2010}. 

In an ideal application, one would select a distance function with good theoretical \emph{and} practical characteristics. Surprisingly, however, we are not aware of any distance metric that performs well in high-dimensional spaces and, conversely, no distance that performs well in high-dimensional spaces belongs to the category of metrics. One is therefore left with a balancing act between performance accuracy and theoretical guarantees, a choice that ultimately hinges on a practitioner's intuition and experience. This work aims to address this situation by proposing a class of distance metrics that, among other benefits, also perform well in high-dimensional spaces. 

\subsection{Contributions.} As discussed above, the motivation for this work is to addresses important needs of a typical data mining pipeline through theoretical and practical contributions. In particular,
\begin{quote} 
($i$) We introduce a new class of distance metrics across different data types, including sets, vector spaces, integrable functions, and ontologies.

\vspace{2mm}
\noindent ($ii$) We identify several important special cases of these metrics, also across different data types. This unexpected unification provides new insights and connections between data mining applications.

\vspace{2mm}
\noindent ($iii$) We analyze theoretical properties of the new metrics and show connections with the Minkowski family and \emph{f}-divergences. This analysis gives inequalities that can be used to provide guarantees in further theoretical studies (e.g., lower risk bounds). 

\vspace{2mm}
\noindent ($iv$) We empirically evaluate the performance of the new metrics against many other distance functions. While our metrics fare well on all types of data, the major distinction is shown on high-dimensional data in text mining applications. 

\vspace{2mm}
\noindent ($v$) We extend the class of distance metrics to ontologies (directed acyclic graphs) drawn from a fixed probability distribution. We demonstrate that these metrics have natural information-theoretic interpretation and can be used for evaluation of classification models in structured-output learning; in particular, when the output of a classifier is a subgraph of a large directed acyclic graph.

\vspace{2mm}
\noindent ($vi$) We evaluate distance metrics on ontologies in two bioinformatics case studies. The first application demonstrates the intuitive nature of new distances by comparing protein sequence similarity against similarity of their molecular and biological functions. The second application clusters several species based solely on biological functions of their proteins, defined via Gene Ontology \cite{Ashburner2000} annotations, and shows that such clustering can recover the evolutionary species tree obtained from protein sequences.
\end{quote} 

\section{Theoretical Framework}
\label{methods}




\subsection{Background}
Metrics are a mathematical formalization of the everyday notion of distance \cite{Goldfarb1992}. Given a non-empty set $X$, a function $d:X\times X\rightarrow\mathbb{R}$ is called a \emph{metric} if
\begin{enumerate}
\item $d(a,b) \geq 0$ (nonnegativity)
\item $d(a,a) = 0$ (reflexivity)
\item $d(a,b) = 0 \; \Leftrightarrow \; a=b$ (identity of indiscernibles)
\item $d(a,b) = d(b,a)$ (symmetry)
\item $d(a,c) \leq d(a,b) + d(b,c)$ (triangle inequality)
\end{enumerate}
\noindent for all $a,b\in X$. A non-empty set $X$ endowed with a metric $d$ is called a \emph{metric space} \cite{book}. 

Although these conditions do not provide the minimum set that defines a metric (e.g., 1 follows from 4 and 5), they are stated to explicitly point out important properties of distance functions and enable us to distinguish between various types of distances. For example, there exists a historical distinction between the general notion of distance (conditions 1, 2, and 4) and that of a  metric \cite{book}, though there are inconsistencies in the more recent literature. Examples of distances that do not satisfy metrics requirements include cosine distance, fractional $L^p$ distances, one minus a Pearson's correlation coefficient, etc. Furthermore, functions such as some $f$-divergences may not even satisfy the symmetry requirement and are generally referred to as dissimilarities or divergences.


In Sections \ref{metrics_on_sets}-\ref{metrics_on_functions} we will introduce new metrics on sets, vectors, and integrable functions. Each metric will have a real-valued parameter $p \geq 1$, with the possibility that $p = \infty$. All proofs can be found in the Appendix.

\subsection{Metrics on sets}
\label{metrics_on_sets}
We start with the simplest case and define two new metrics on finite sets. Both will be extended to more complex situations in subsequent sections.

\subsubsection{Unnormalized metrics on sets}
Let $X$ be a non-empty set of finite sets drawn from some universe. We define a function $\d{p}:X\times X\rightarrow\mathbb{R}$ as
\begin{equation}
\d{p}(A,B) = (|A-B|^p+|B-A|^p)^\frac{1}{p},\label{eq:1}
\end{equation}
\noindent where $|\cdot|$ denotes set cardinality, $A-B = A \cap B^c$, and $p \geq 1$ is a parameter mentioned earlier.
\begin{thm}
$(X, d^p)$ is a metric space.
\label{ThmUnnormalizedSet}
\end{thm}
The symmetric distance on sets is a special case of $\d{p}$ when $p=1$ and it converges to the bag distance as $p \rightarrow \infty$ \cite{book}. 


\subsubsection{Normalized metrics on sets}
Let $X$ again be a non-empty set of finite sets drawn from some universe. We define a function $\dN{p}:X\times X\rightarrow\mathbb{R}$ as
\begin{equation}
\dN{p}(A,B) = 
\dfrac{(|A-B|^p+|B-A|^p)^\frac{1}{p}}{|A\cup B|}, \label{eq:2}
\end{equation}
\noindent if $|A\cup B| \neq 0$ and zero otherwise. 

\begin{thm}
\label{ThmSet}
$(X, \dN{p})$ is a metric space. In addition, $\dN{p}:X \times X \rightarrow[0,1]$.
\end{thm}

Observe that the Jaccard distance is a special case of $\dN{p}$ when $p=1$. 
%
%



\subsubsection{Relationship to Minkowski distance} 
Although the new metrics have a similar form to the Minkowski distance on binary set representations, they are generally different. Take for example $A=\left\{ 1,2,4\right\} $ and $B=\left\{ 2,3,4,5\right\}$ from a universe of $k=5$ elements. A sparse set representation results in the following encoding: $\mathbf{a}=\left(1,1,0,1,0\right)$ and $\mathbf{b}=\left(0,1,1,1,1\right)$. The Minkowski ($L^p$) distance of order $p$ between $\mathbf{a}$ and $\mathbf{b}$ is defined as
\[
d^p_{\textrm{M}}(\mathbf{a},\mathbf{b})=\left({\textstyle \sum_{i=1}^{k}}\left|a_i-b_i\right|^{p}\right)^{\nicefrac{1}{p}} = \left\Vert \mathbf{a}-\mathbf{b}\right\Vert _{p},
\]
\noindent and $p\geq1$. Substituting the numbers into the expressions above gives $d_{\textrm{M}}^{1}(\mathbf{a},\mathbf{b})=3$ and $\d{1}(A,B)=3$; $d^{2}_{\textrm{M}}(\mathbf{a},\mathbf{b})=\sqrt{3}$ and $\d{2}(A,B)=\sqrt{5}$, etc. In fact, $d_{\textrm{M}}^{p}(\mathbf{a},\mathbf{b})\neq \d{p}(A,B)$ for all $p > 1$.


\subsection{Metrics on vector spaces}
\label{metrics_on_vectors}
We define a version of our metrics on the vector space $\R^k$, where $k \in \mathbb{N}$ is the dimension of the space. Let $\mathbf{x}=(x_1,x_2, \ldots,x_k)$ and $\mathbf{y}=(y_1,y_2, \ldots,y_k)$ be any two points in $\R^k$.

\subsubsection{Unnormalized metrics on vectors}
We define a function $\d{p}: \R^k \times \R^k \rightarrow \mathbb{R}$ as 
\begin{equation}
\d{p}(\mathbf{x},\mathbf{y}) = \Big(\big(\sum_{i:x_i \geq y_i}x_{i}-y_{i}\big)^{p}+\big(\sum_{i:x_i<y_i}y_{i}-x_{i}\big)^{p}\Big)^{\frac{1}{p}}.
\label{eq:7}
\end{equation}
\begin{thm}
$(\R^k, \d{p})$ is a metric space.
\label{ThmVectorU}
\end{thm}
When $p=1$ the distance from Eq.~\ref{eq:7} is equivalent to the Manhattan (cityblock) distance.

\subsubsection{Normalized metrics on vectors}
We define a function $\dN{p}: \R^k \times \R^k \rightarrow \mathbb{R}$ as 
\begin{equation}
\dN{p}(\mathbf{x},\mathbf{y})=\frac{\d{p}(\mathbf{x},\mathbf{y})}{\sum_{i=1}^{k}\max(|x_{i}|,|y_{i}|,|x_{i}-y_{i}|)}.\label{eq:8}
\end{equation}
\begin{thm}
$(\R^k,\dN{p})$ is a metric space. In addition, $\dN{p}:X \times X \rightarrow[0,1]$. 
\label{ThmVectorN}
\end{thm}
As mentioned above, the new metrics $\d{p}$ and the Minkowski distance $d_{\textrm{M}}^p$ on $\R^k$ are different for $p>1$. However, we were able to establish a \emph{strong equivalence} between the two in Section \ref{connection}. Therefore, many useful properties of the class of Minkowski distances also hold for the metrics $\d{p}$. For instance, the completeness of $(\R^k,d_{\textrm{M}}^p)$ implies that $(\R^k,\d{p})$ is also complete.

We alert the reader that we used the same symbol $\d{p}$ in Eqs.~\ref{eq:1} and \ref{eq:7} and $\dN{p}$ in Eqs.~\ref{eq:2} and \ref{eq:8}, but believe it should not present interpretation problems. For example, Eq.~\ref{eq:8} is not the Jaccard distance when $p=1$, but rather its analog in real-valued vector spaces, as defined in this work. We shall continue this notation pattern in the next section.

\subsection{Metrics on integrable functions}
\label{metrics_on_functions}
We now extend the previously introduced metrics to integrable functions and show that the space of the integrable functions equipped with the new metrics is complete. 

\subsubsection{Unnormalized metrics on functions}
Let $L(\R)$ be a set of integrable functions on $\mathbb{R}$. We define $\d{p}: L(\R) \times L(\R) \rightarrow \mathbb{R}$ as
\begin{equation}
    \d{p}(f,g) = \left((\int(f-g)^{+}\,dx)^p+(\int(f-g)^{-}\,dx)^p\right)^\frac{1}{p},
\label{eq:5}
\end{equation}
\noindent where $f^+ = \max(f,0)$, $f^- = \max(-f,0)$.

\begin{thm}
$(L(\R),\d{p})$ is a metric space.
\label{Thm1}
\end{thm}
The well-known $L^1$ distance is a special case of $\d{p}$ when $p=1$.

\subsubsection{Normalized metrics on functions}
Let $L(\R)$ again be a set of bounded integrable functions on $\mathbb{R}$ and $\d{p}$ the distance function from Eq.~\ref{eq:5}. We define $\dN{p}: L(\R) \times L(\R) \rightarrow \mathbb{R}$ as

\begin{equation}
    \dN{p}(f,g) = \dfrac{\d{p}(f,g)}{\int \max(|f|,|g|,|f-g|)\,dx}.
 \label{eq:6}
\end{equation}
%

\begin{thm}
$(L(\R),\dN{p})$ is a metric space. In addition, $\dN{p}:L(\mathbb{R})\times L(\mathbb{R})\rightarrow[0,1]$. 
\label{ThmR}
\end{thm}

Observe that the Marczewski-Steinhaus distance~\cite{marczewski1958certain} is a special case of $\dN{p}$ when $p=1$.

\begin{thm}
$(L(\R),\d{p})$ and $(L(\R),\dN{p})$ are complete metric spaces.
\label{completeness}
\end{thm}

\subsubsection{Geometric interpretation of the new distances}
We illustrate the geometry of new distances in Figure \ref{fig:geometry}. Consider two functions $f(x)$ and $g(x)$. Let $A_{f>g}$ be the total area (volume) of the space between $f$ and $g$ where $f > g$ and $A_{f<g}$ be the area where $g > f$. Our unnormalized distance corresponds to the $L^p$ norm of the vector $(A_{f>g}, A_{f<g})$. A standard $L^p$ distance between $f$ and $g$ would instead integrate $|f - g|^p$ and be more sensitive to isolated large differences. Our metrics mitigate this effect through dimension reduction to $\mathbb{R}^2$. The similarity between $f$ and $g$ depends on the balance of $A_{f>g}$ and $A_{f<g}$ as a function of $p$.

\begin{figure}[t!]
\centering
    \includegraphics[width=0.5\columnwidth]{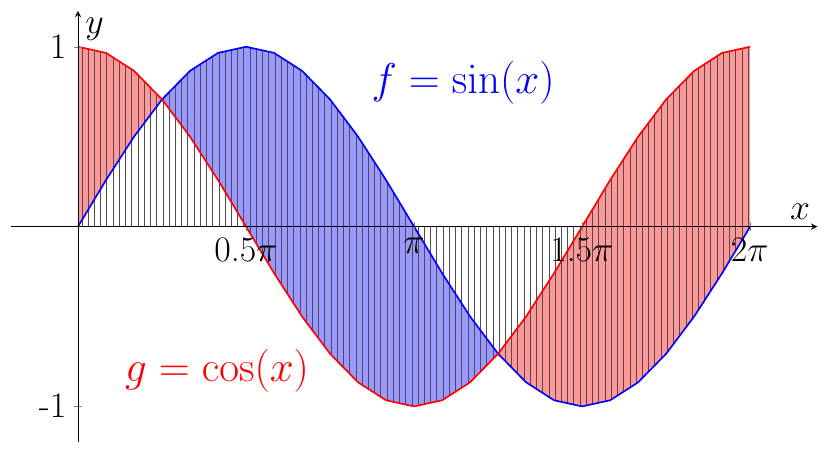}
    \caption{Geometry of the new distances between two functions $f(x)=\sin (x)$ and $g(x)=\cos (x)$ over $[0,2\pi]$. The blue area corresponds to $A_{f>g}$, whereas the red area corresponds to $A_{f<g}$. The vertical lines visualize the normalization factor from Eq.~\ref{eq:6}.}
    \label{fig:geometry}
\end{figure}

\section{Connections with other dissimilarity measures}
\label{connection}

\subsection{Equivalence with Minkowski distances}
In the previous section we noted that our various metrics reduce to certain well-known metrics when the parameter $p$ is either $1$ or $\infty$.
In the case of our unnormalized metrics on vector spaces from Eq.~\ref{eq:7}, we can establish a stronger relationship with the Minkowski distance for $p\geq1$ and $p=\infty$.
%
\begin{prop}
\label{equivalence}
The new metric $\d{p}$ and the Minkowski distance $d_{\textrm{M}}^p$ on $\R^k$,  when they share the same parameter $p$ where $p\geq 0$ or $p=\infty$, are \emph{equivalent metrics}; that is, there exist positive constants $\alpha$ and $\beta$ such that for all $\mathbf{x},\mathbf{y}\in \R^k$ it holds that
\begin{equation}
\label{equivalence_metrics}
\alpha \d{p}(\mathbf{x},\mathbf{y}) \leq d_{\textrm{M}}^p(\mathbf{x},\mathbf{y}) \leq \beta \d{p}(\mathbf{x},\mathbf{y}).
\end{equation}
\end{prop}
This proposition can be proved by invoking H\"{o}lder's inequality and some algebraic manipulations (Appendix). 

\subsection{Comparisons with $f$-divergences for probability distributions}

Suppose $P$ and $Q$ are probability distributions for some random variables defined on a Lebesgue-measurable set in $\mathbb{R}$ with probability densities $h$ and $g$ in $L(\R)$ respectively. Then $\d{p}(h,g)$ or $\dN{p}(h,g)$ provide a measure of dissimilarity between $h$ and $g$. In comparison, an $f$-divergence of $P$ with respect to $Q$ is the expectation of $f(dP/dQ)$ under the distribution $Q$ \cite{csiszar1967} with regularity constraints on $f$. Replacing $f(t)$ by $\frac{1}{2}|t-1|$, $(1-\sqrt{t})^2$ or $t\log(t)$ gives the total variation $\operatorname{TV}(P,Q)$, the Hellinger distance $H(P,Q)$, or the Kullback-Leibler~(KL) divergence $D_{\textrm{KL}}(P||Q)$, respectively \cite{Liese2006}.

The total variation and the Hellinger distance are metrics while the KL divergence is not. However, the KL divergence $D_{\textrm{KL}}(P||Q)$ is meaningful as it measures the information theoretic divergence when $P$ is the true underlying distribution for the model in hand and $Q$ is the presumed distribution in model development. The information theoretic bounds on compressibility and relationship to maximum-likelihood inference of $D_{\textrm{KL}}(P||Q)$ are well understood \cite{Cover2006}.

When $p=1$, our distance on probability densities $h$ and $g$ is equivalent to the total variation of their distributions $P$ and $Q$ as $\d{p}(h,g) = 2\operatorname {TV}(P,Q)$. Based on this equality we are able to establish the relationships between $\d{p}(h,g)$ to their corresponding KL divergence and Hellinger distance.

\begin{prop}
\label{KL}
Let $P$ and $Q$ be probability distributions with respect to some real random variables with probability densities $h$ and $g$ in $L(\R)$, respectively. For any $p\geq 1$ it holds that
\[
\d{p}(h,g)\leq \sqrt{2\min (D_{\textrm{KL}}(P||Q),D_{\textrm{KL}}(Q||P))}.\]
\end{prop}
The result directly follows from Pinsker's inequality \cite{Pinsker1964} and the symmetry of metrics. Interestingly, the converse does not hold. That is, there exist sequences of probability density functions $\{h_n\}$ and $\{g_n\}$ such that $\d{p}(h_n,g_n) \rightarrow 0$ but $D_{\textrm{KL}}(P_n||Q_n) \rightarrow \infty$. 
\begin{prop}
\label{Hellinger}
Under the same conditions as in Proposition \ref{KL}, it holds that
\[
2H(P,Q)^2 \leq \d{p}(h,g)\leq 2\sqrt{2} H(P,Q).\]
\end{prop}
The conclusion follows from $H(P,Q)^2 \leq \operatorname{TV}(P,Q)\leq  \sqrt{2}H(P,Q)$; see \cite{lecam1973convergence}.

Proposition \ref{KL} and Proposition \ref{Hellinger} (up to a multiplicative constant) also apply to $\dN{p}(h,g)$ as $\frac{1}{2}\d{p}(h,g) \leq \dN{p}(h,g)\leq \d{p}(h,g)$ since $h$ and $g$ are densities. These inequalities can prove useful in establishing lower risk bounds in applications that directly minimize the new distances as opposed to $f$-divergences \cite{Guntuboyina2011}.

\section{Empirical Investigation} \label{experiments}
We evaluate the performance of the new metrics on real-valued and text document data using classification problems. We then compare them against a wide range of dissimilarity measures. Classification on each data set was carried out by applying the KNN algorithm \cite{Cover1967} with $K$ taking all values from $\{1, 3, 5, \ldots, \lceil\sqrt{n}\,\rceil \}$, where $n$ is the data set size. In all experiments we conducted a leave-one-out accuracy estimation and compared distance functions based on the maximum accuracy achieved over all $K$. To compare distances $d_1$ and $d_2$ on a particular data set, we scored a ``win'' to the one with the larger maximum accuracy or assigned half a win to each in case of a tie. We then counted the number of wins in a ``tournament'' where each distance was pairwise-compared with all its competitors on each data set. The expectation is that a better distance will lead to higher classification performance and more wins.

The new metrics were compared to the following distances: (1) Minkowski ($L^p$) family; (2) fractional $L^p$ distances; i.e., Minkowski distances with $0<p<1$; (3) normalized $L^p$ distances; i.e., $\left\Vert \mathbf{x}-\mathbf{y}\right\Vert _{p}/(\left\Vert \mathbf{x}\right\Vert _{p}+\left\Vert \mathbf{y}\right\Vert _{p})$; (4) cosine distance; (5) one minus Pearson's and Spearman's correlation coefficient; (6) $f$-divergences; i.e., Hellinger distance and total variation; and (7) the fractional equivalent of all our metrics ($f^p$ and $f^p_N$, $0<p<1$). Note that all Minkowski distances, Hellinger distance, and total variation (applicable only to nonnegative inputs) are metrics. However, all fractional distances, all normalized $L^p$ distances (except when $p=2$, when it is a metric \cite{book}), cosine distance, and one minus the correlation coefficient are not metrics. For example, the cosine distance on $\mathbb{R}^k - \left \{ 0 \right \}^k$ violates the identity of indiscernibles and the triangle inequality. 
For all $L^p$ and $d^p$ distances, we varied $p$ from $\{1, 2, 4, 8, 16, 32, \infty \}$ and for fractional distances from $\{ \nicefrac{1}{2}, \nicefrac{1}{4}, \nicefrac{1}{8}, \nicefrac{1}{16}, \nicefrac{1}{32} \}$.

\vspace{0.2cm}
\subsection{Evaluation results}
The evaluation on real-valued data was performed over thirty z-score normalized data sets from the UCI Machine Learning Repository \cite{Lichman2013}; summarized in the Appendix. Figure~\ref{fig:wins_real_max} shows the number of wins per data set with the variation assessed by bootstrapping. That is, the set of data sets was sampled with replacement 1000 times from which wins and losses were counted as described above. The top performing distance measure on these relatively low-dimensional problems was the cityblock distance, which belongs to both $L^p$ and $d^p$ family when $p=1$, followed by $d^2$ and $d^2_N$ metrics. Importantly, we also find that the $d^p$ and $d_N^p$ metrics outperform their $L^p$ counterparts when $p > 1$.

\begin{figure*}[ht!]
    \centering
    \includegraphics[width=1.0\columnwidth]{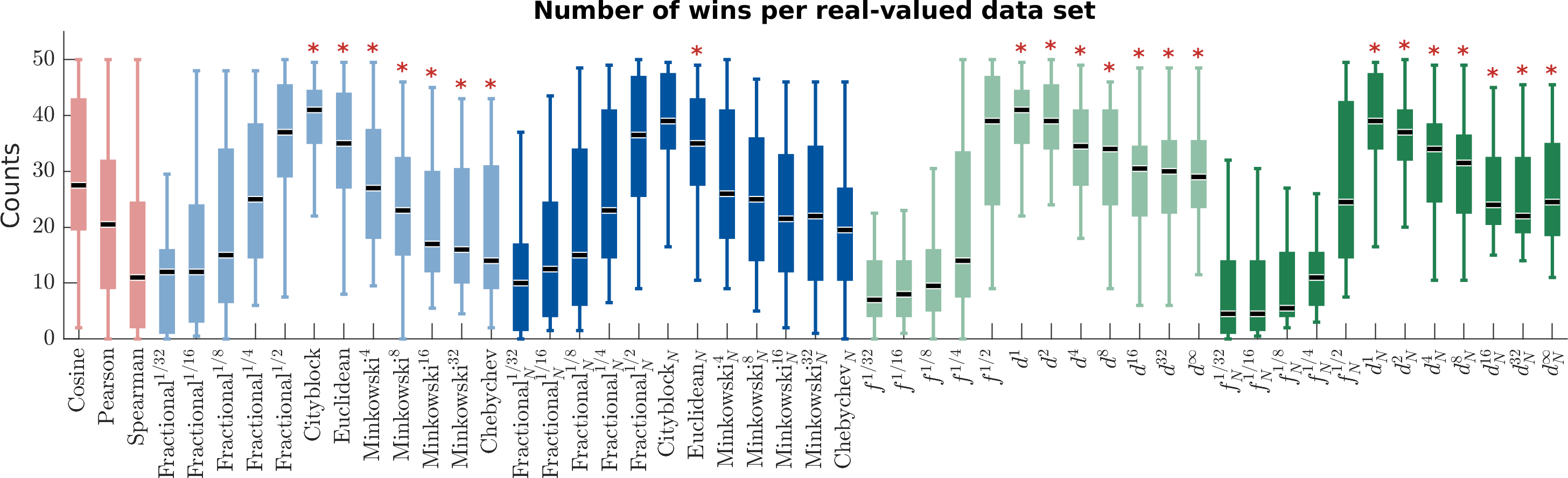}
    \caption{Comparison between dissimilarities. The functions are color coded as follows: $L^p$ family (light blue), normalized $L^p$ family (dark blue), $d^p$ family (light green), normalized $d^p$ family (dark green), cosine distance and the correlation coefficients (light red). All metrics are labeled by an asterisk.}
    \label{fig:wins_real_max}
\end{figure*}
%


The performance evaluation on high-dimensional data was carried out using tf-idf encoding \cite{Tan2006} on ten text document data sets. The first data set was constructed in this work by using abstracts from five life sciences journals with the task of predicting the journal each paper was published in. The remaining data sets included Cade12 and webkb from \cite{Cachopo2007}; 20NewsGroups downloaded via scikit-learn library; MovieReview from \cite{Pang2004}; CNAE9, Reuters and TTC-3600 from the UCI Machine Learning Repository and two data sets extracted from the literature \cite{Dalkilic2006, greene06icml}. 
Each data set was treated as a multi-class classification problem; see Appendix.

Figure~\ref{fig:wins_text_max} shows the number of wins per text data set. Here we find that normalized distance functions outperformed their unnormalized counterparts. In fact, as a group, the $d_N^p$ metrics show the best performance, with the maximum reached when $p=2$. The $d_p$ and $d_N^p$ metrics again outperformed their $L^p$ counterparts. Interestingly, the normalized $L^1$ distance, cosine distance and one minus correlation coefficients show excellent performance on tf-idf data. However, neither of these functions is a metric. Therefore, the $d_N^p$ metrics are the only group that provide both high performance accuracy and theoretical guarantees reserved for metric spaces. Specifically, $d_N^2$ performs significantly better than the normalized Euclidean metric on these high dimensional datasets ($P = 0.0107$; Bionomial test). Additional results obtained by varying the performance evaluation criteria (mean vs.~maximum accuracy over all $K$), data types (sparse binary sets), and data normalization procedures (z-score, min-max, unit) are provided in the Appendix.


\begin{figure*}[ht!]
    \centering
    \includegraphics[width=1.0\columnwidth]{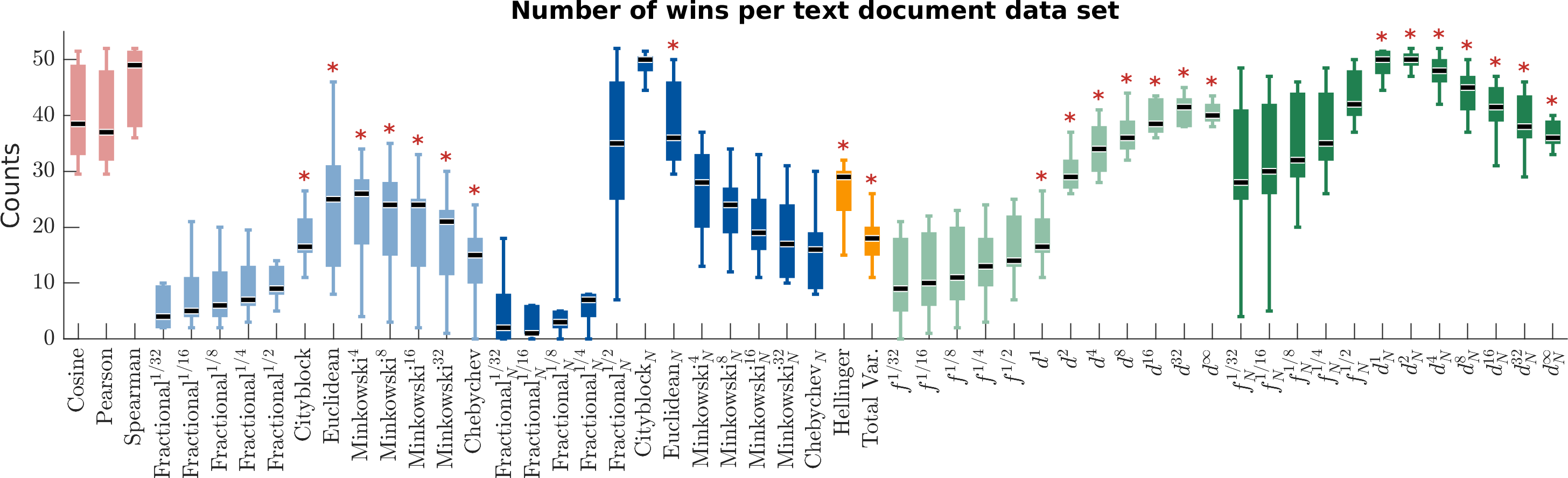}
    \caption{Comparison between dissimilarities. The functions are color coded as follows: $L^p$ family (light blue), normalized $L^p$ family (dark blue), $d^p$ family (light green), normalized $d^p$ family (dark green); $f$-divergences (light yellow), cosine distance and the correlation coefficients (light red). All metrics are labeled by an asterisk.}
    \label{fig:wins_text_max}
\end{figure*}


Overall, the following results stand out. First, metrics have generally outperformed non-metrics and fractional distances did not provide the expected improvement on high-dimensional data. The $d^p$ distances outperformed their $L^p$ counterparts over all $p>1$. As expected, the cosine distance worked well on high-dimensional data, but surprisingly the two correlation coefficients were as good. Finally, when averaged over the two groups of data (Figures \ref{fig:wins_real_max}-\ref{fig:wins_text_max}), $d^2_N$, $d^1_N$, and normalized $L^1$ distances were the top three performing distance functions. Combined with the fact that the normalized $L^1$ distance is not a metric, these results provide compelling evidence that the new metrics fare well against all competing distances.

\subsection{Hubness in high-dimensional spaces}
Recent work has shown that high-dimensional data sets suffer from the effect of hubness. That is, as the data dimensionality increases, a smaller fraction of points tend to find themselves as nearest neighbors of many other points in the data set, whereas a larger fraction of points tend to be within no one's nearest neighbors \cite{radovanovic2010hubs}. This effect is related to the poor performance of traditional distance functions in high-dimensional spaces \cite{beyer1999nearest}. 

We investigated the effect of hubness for the $d^p$ and $d^p_N$ metrics on both simulated and real data. Following \cite{radovanovic2010hubs}, for every data point $\mathbf{x}$ we first counted the number of other points in the data set such that $\mathbf{x}$ was within their $K$ closest neighbors, $N_K(\mathbf{x})$. We then plot the distribution of $N_5$ on an i.i.d.~Gaussian data and two real-life data sets from the collection used in this work (Figure \ref{fig:hubness}). We find that the $d^p_N$ metrics show similar hubness effects and resilience to high dimension as the cosine distance, while at the same time being a metric. There do not exist similar normalizers for the $L^p$ family, except when $p=2$ \cite{book}.


\subsection{Performance insights}
A potential factor contributing to the success of the $\dN{p}$ metrics on the tf-idf data could be the lack of translational invariance. We call a metric $d$ on $X$ \textit{translation-invariant} if $d(x,y) = d(x+z,y+z)$ for all $x,y,z \in X$. A number of classical metrics fall into this group, such as all norm-induced metrics; e.g., Minkowki distances. The normalized metric $\dN{p}$ is not translation-invariant as demonstrated by a simple example that $\forall p>1, \dN{p}(0,1)\neq \dN{p}(1,2)$. This effect, however, is important for quantifying distance in semantic data such as text documents. To illustrate this, consider the following bag-of-words features of two pairs of article abstracts in $\R^k$ for $p=1$;  $\dN{p}((1,0,\ldots,0),(0,\ldots, 0)) =1$ while $\dN{p}((100,100,\ldots,100),(99,100,\ldots,100)) = 1/(100k)$. The two pairs of elements have equal Minkowski distance of $1$, but the elements $(100,100,\ldots,100)$ and $(99,100,\ldots,100)$ are more related as they share a large number of words. The normalized metric $\dN{p}$ captures that strong similarity by incorporating built-in information of the data as indicated by the results on text data (Figure \ref{fig:wins_text_max}). The tf-idf data studied here can also be viewed as ontological data with a trivial structure, the concept of which will be introduced in Section \ref{ontology}.


\subsection{Computational efficiency}
Computing the Minkowski distance of order $p$ between two $k$-dimensional vectors requires $2k-1$ additions and $k$ exponentiations, before calculating the $p$-th root. Our unnormalized metric from Eq.~\ref{eq:7} requires $2k-1$ additions, $k$ comparisons to a $0$, and only $2$ exponentiations. Since exponentiation is slow, especially for larger $p$, the new metric is faster to compute. Both classes of metrics have the same asymptotic complexity of $O(k)$.

\begin{figure}
    \centering
    \includegraphics[width=.45\linewidth]{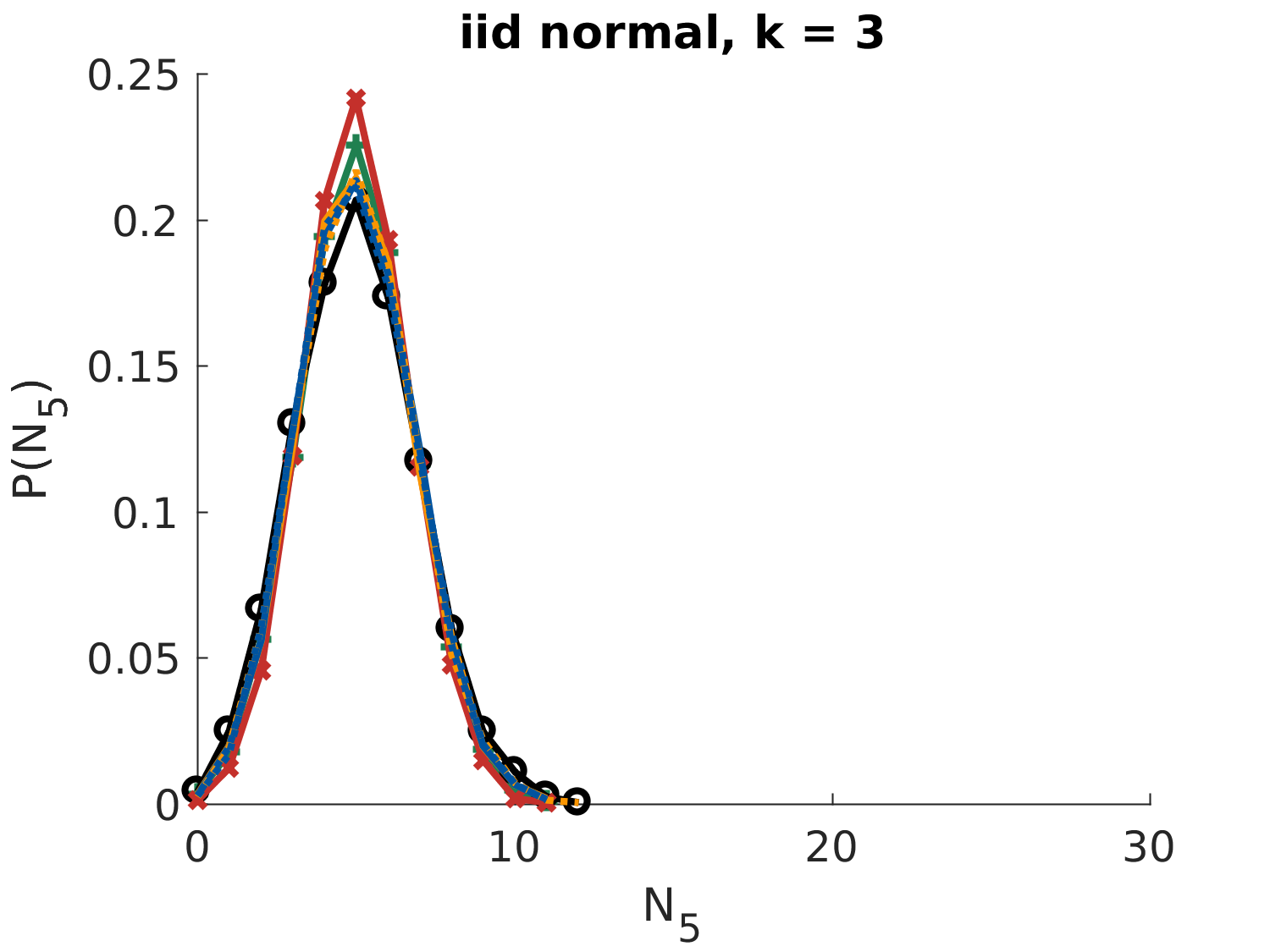}
    \includegraphics[width=.45\linewidth]{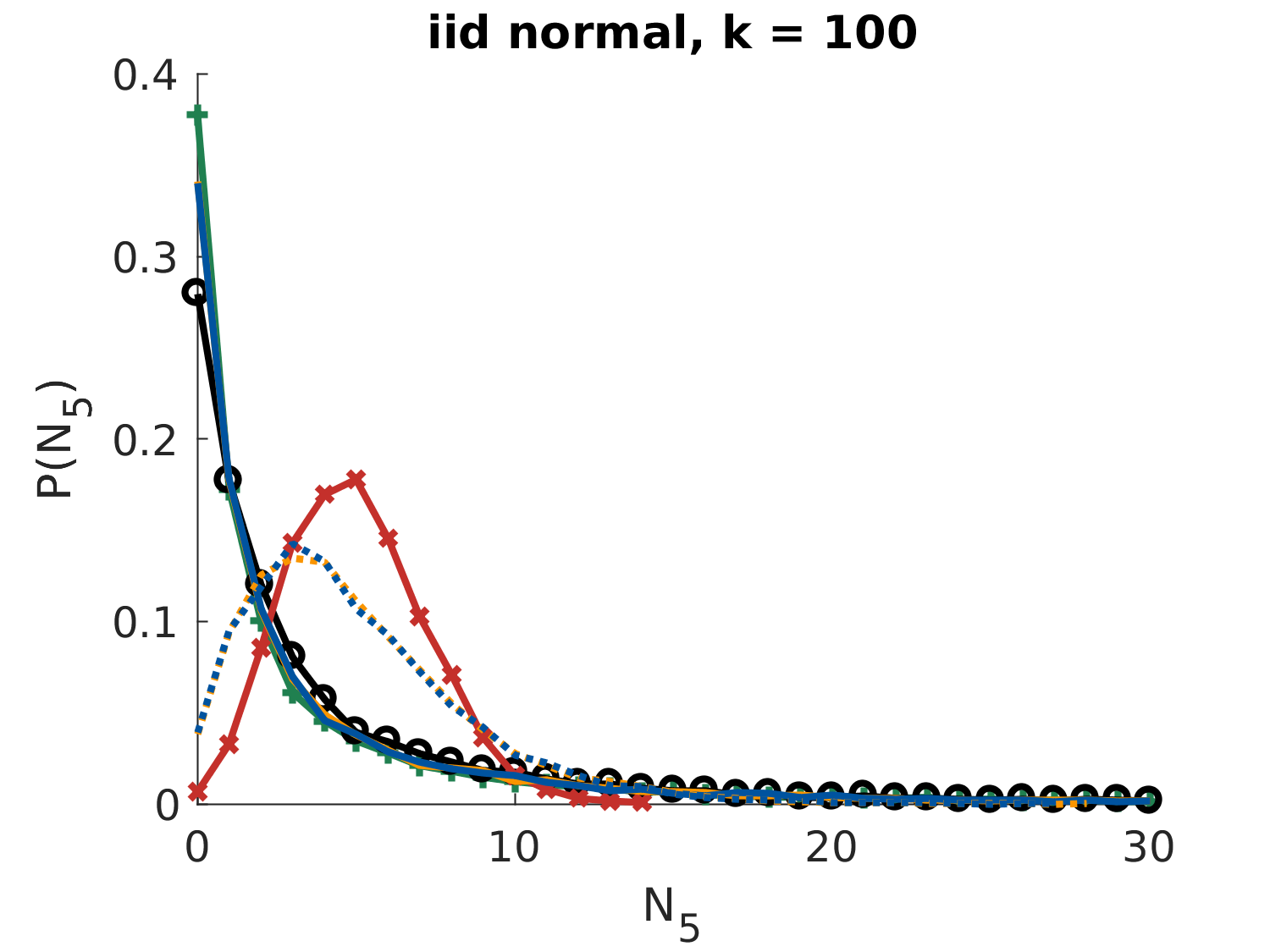} \\
    \includegraphics[width=.45\linewidth]{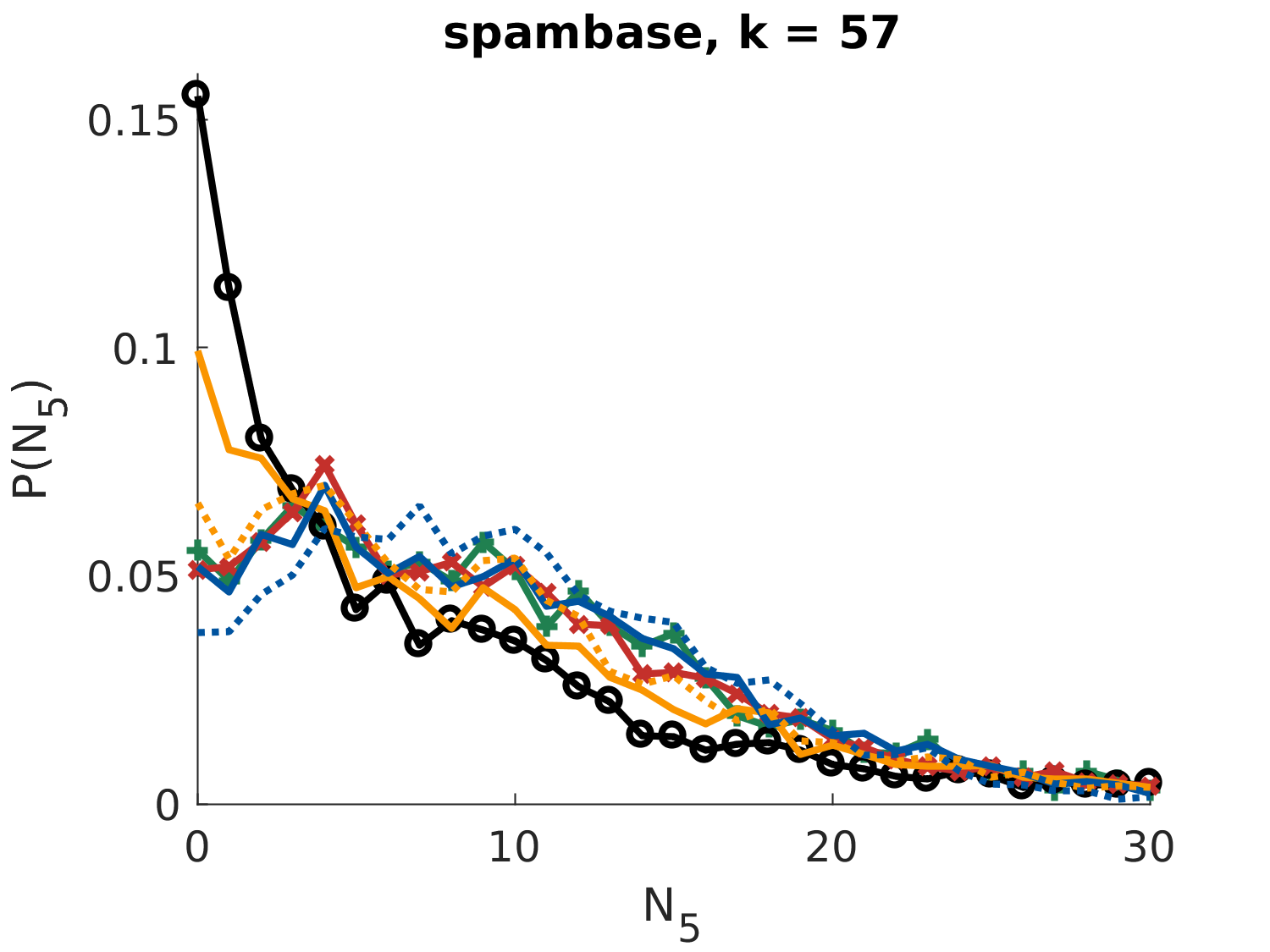}
    \includegraphics[width=.45\linewidth]{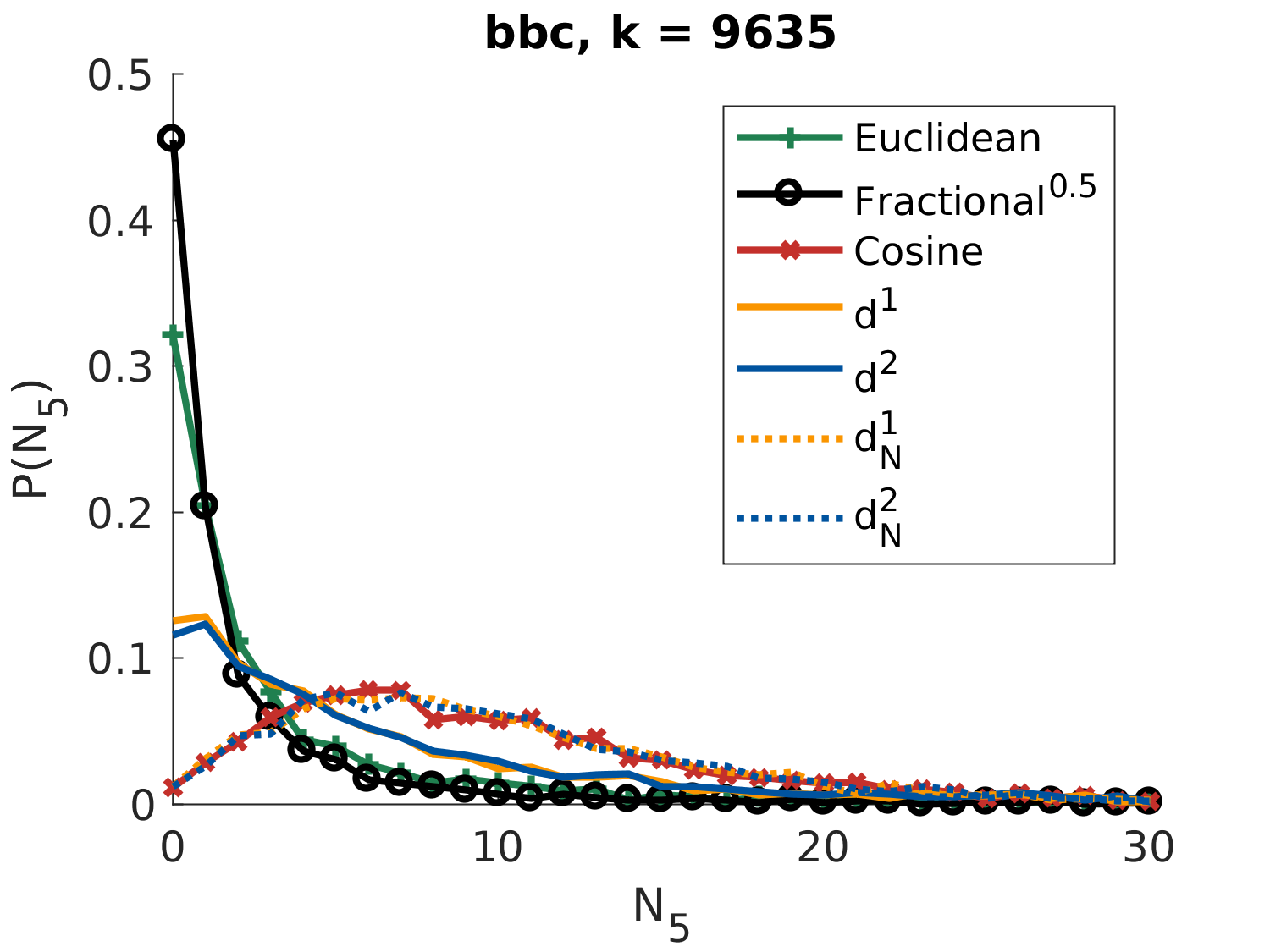}
    \caption{Hubness in low- and high-dimensional data sets. The upper panels show the distribution of $N_5$ for simulated data sets with the standard normal distribution ($n=10000$, $k \in \left \{3, 100\right \}$). The lower panels show the distribution of $N_5$ for the \textit{spambase}  ($n=4601$, $k = 57$) and \textit{bbc} ($n=2225$, $k = 9635$) data sets.}
    \label{fig:hubness}
\end{figure}
\section{Application to ontologies}
\label{ontology}
Modern classification approaches increasingly rely on ontological output spaces \cite{Grosshans2014,Movshovitz2015}. An ontology $\mathcal{O}=(V,E)$ is a directed acyclic graph with a set of vertices (concepts) $V$ and a set of edges (relational ties) $E\subset V\times V$. A news article, for instance, covering aspects of sports injuries might be labeled by the term ``sports'', ``medicine'', but also ``sports medicine'' that is a subcategory of both sports and medical articles. Similarly, a protein associated with the terms ``transferase'' and ``oxidoreductase'' could also be associated with a more general term ``enzyme''. In terms of class labels, a news article or a protein function can be seen as a \emph{consistent subgraph} $F\subseteq V$ of the larger ontology graph. By saying consistent, we mean that if a vertex $v$ belongs to $F$, then all the ancestors of $v$ up to the root(s) of the ontology must also belong to $F$. This consistency requirement follows from the transitive relationships specified on edges that are commonly used; e.g., \textsf{is-a} and \textsf{part-of}. In some domains such as computational biology, a subgraph $F$ corresponding to an experimentally characterized protein function contains 10-100 nodes, whereas the ontology graph consists of 1000-10000 nodes \cite{Robinson2011}. 

Before we introduce metrics on ontological annotations, we briefly review relevant theoretical concepts. Suppose that the underlying probabilistic model according to which ontological annotations have been generated is a Bayesian network structured according to the ontology $\mathcal{O}$. That is, we consider that each concept in the ontology is a binary random variable and that the graph structure specifies the conditional independence relationships in the network. Then, using the standard Bayesian network factorization we write the marginal probability for any consistent subgraph $F$ as
\[
P(F)=\prod_{v\in F}P(v|\mathrm{Parents}(v)),
\]
\noindent where $P(v|\mathrm{Parents}(v))$ is the probability that node $v$ is part of an ontological annotation given that all of its parents are part of the annotation. Due to consistency, the marginalization can be performed in a straightforward manner from the leaves of the network towards the root. This marginalization is reasonable in open-world domains such as molecular biology because some activities are never tested and those that are might not be fully observable. Thus, treating nodes not in $F$ as unknown and marginalizing over them is intuitive. Observe that each conditional probability table in this (restricted) Bayesian network needs to store a single number; i.e., the concept $v$ can be present only if all of its parents are part of the annotation. If any of the parents is not a part of the annotation $F$, $v$ is guaranteed to not be in $F$.

\subsection{Metrics on ontologies}
We express the information content of a consistent subgraph $F$ as
\noindent 
\begin{align*}
i(F) & =\log\frac{1}{P(F)}
  ={\textstyle \sum}_{v\in F}ia(v),
\end{align*}
%
%
\noindent where $ia(v) = -\log P(v|\mathrm{Parents}(v))$ is referred to as information accretion \cite{Clark2013}. This term corresponds to the additional information inherent to the node $v$ under the assumption that all its parents are already present in the annotation of the object.
%
%
%
%
%

We can now compare two ontological annotations $F$ and $G$. For the moment, suppose that annotation $G$ is a prediction of $F$. We use the term \emph{misinformation} to refer to the cumulative information content of the nodes in $G$ that are not part of the true annotation $F$; i.e., it gives the total information content along all incorrect paths in $G$. Similarly, the \emph{remaining uncertainty} gives the overall information content corresponding to the nodes in $F$ that are not included in the predicted graph $G$ (Figure \ref{fig:comparison}). More formally, misinformation and remaining uncertainty are defined as
\[
mi(F,G)=\sum_{v\in F-G}ia(v)\:\:\: \textrm{and} \:\:\: ru(F,G)=\sum_{v\in G-F}ia(v).
\]
%
\begin{figure}
\begin{centering}
\includegraphics[width=0.65\columnwidth]{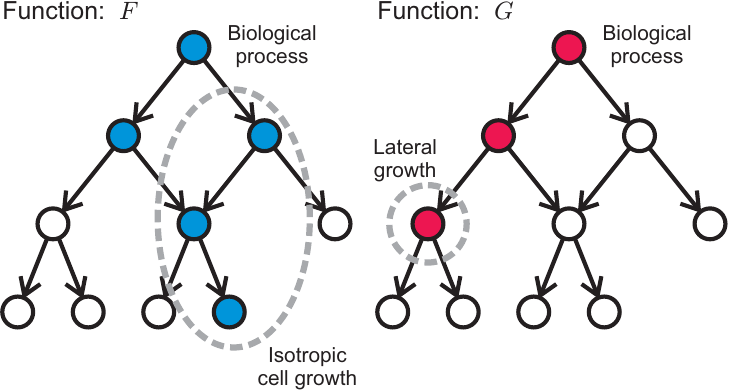}
\par
\end{centering}
\protect\caption{Illustration of the calculation of the remaining uncertainty and misinformation for two proteins with their ontological annotations: $F$ (true, blue) and $G$ (predicted, red). The circled nodes contribute to the remaining uncertainty (blue nodes, left) and misinformation (red node, right).}
\label{fig:comparison}
\end{figure}

\noindent Let now $X$ be a non-empty set of all consistent subgraphs generated according to a probability distribution specified by the Bayesian network. We define a function $\d{p}:X\times X\rightarrow\mathbb{R}$ as
\begin{equation}
\d{p}(F,G)=\left(ru^{p}(F,G)+mi^{p}(F,G)\right)^{\frac{1}{p}}.\label{eq:3}
\end{equation}
%
\noindent We refer to the function $\d{p}$ as semantic distance. Similarly, we define another function $\dN{p}:X\times X\rightarrow\mathbb{R}$ as
\begin{equation}
\dN{p}(F,G)=\frac{\left(ru^{p}(F,G)+mi^{p}(F,G)\right)^{\frac{1}{p}}}{\sum_{v\in F\cup G}ia(v)}.\label{eq:4}
\end{equation}
%
\noindent We refer to the function $\dN{p}$ as normalized semantic distance.
%
%
\begin{thm}
$(X, \d{p})$ is a metric space.
\label{ThmOntologyU}
\end{thm}
\begin{thm}
$(X, \dN{p})$ is a metric space. In addition, $\dN{p}: X \times X \rightarrow[0,1]$.
\label{ThmOntologyN}
\end{thm}

\subsection{Indirect evaluations using ontological annotations for proteins}
Evaluating dissimilarity measures between consistent subgraphs is difficult because ontological annotations are usually class labels rather than attributes, and this excludes a classification-based evaluation from Section \ref{experiments}. Therefore, we use an indirect approach and assess the quality of the proposed semantic distance between ontological annotations (class labels) using domain knowledge on the set of protein sequences each of which has a consistent subgraph associated with it as a class label.

In the first experiment, we take a set of protein sequence pairs with their class labels from the UniProt database and compare sequence similarity to class-label distance with a biologically justified expectation that the higher sequence similarity results in lower class label dissimilarity. Figure~\ref{fig:semantic_dist_ontologies} shows the relationship between sequence similarity and class-label distance for a set of three separate ontologies in the Gene Ontology \cite{Ashburner2000}. As expected, we observed significant difference for the two groups of sequences in each category. Statistical tests give $P$-values close to zero on each data set.


In the second experiment we perform clustering of five species for which we could extract a sufficient number of class labels. Hierarchical clustering on the groups of ontological annotations (one group for each species) was used to form an evolutionary tree; for simplicity, we refer to the tree derived solely from functional information as a \emph{functional phylogeny}. A good class-label dissimilarity is expected to provide the same evolutionary tree as the one that has been determined by evolutionary biologists based on sequence. 


Using the Molecular Function and Cellular Component functional annotations of the Gene Ontology, our clustering approach did recover the correct relationships among species (Figure \ref{fig:phylogenetic_tree}, left tree). This result is gratifying, especially as we might expect many similar functions to be present in the single-celled organisms (\textit{E.~coli} and \textit{S.~cerevisiae}). 
However, using the Biological Process annotations did not result in the correct phylogeny, as the positions of \textit{S.~cerevisiae} and \textit{A.~thaliana} were reversed (Figure \ref{fig:phylogenetic_tree}, right tree). 

The accuracy of the molecular function and cellular component annotations and the inaccuracy of the biological process annotations are consistent with the higher level of functional conservation for the less abstract annotations \cite{Rogers2009, Nehrt2011}, as greater conservation of function could result in more phylogenetic signal within this ontology. As a reminder, this algorithm only produces an unrooted topology among the species. It is up to the experimenters to root the tree with some expert knowledge, as we have done here.

\begin{figure}
    \centering
    \includegraphics[width=.8\columnwidth]{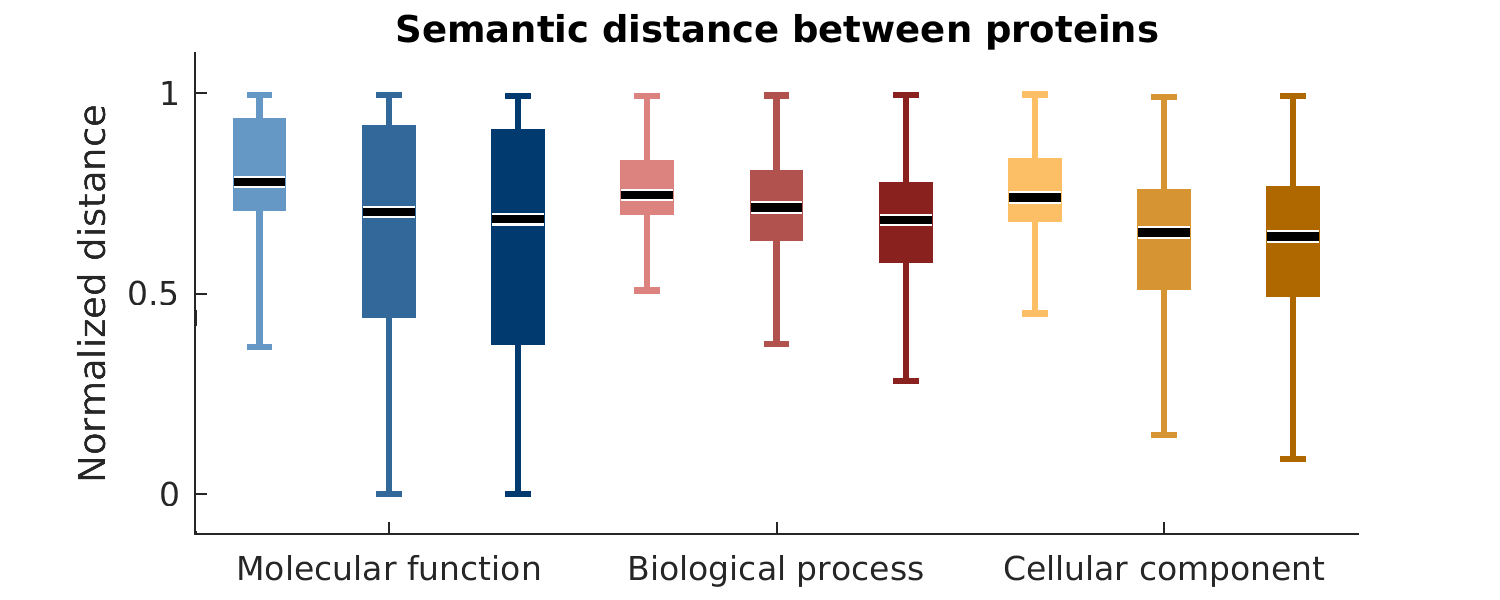}
    \caption{Indirect evaluation of semantic distance between ontological annotations. Distance comparisons are color-coded for the three Gene Ontology domains. In each domain, three boxes from light to dark show the pairwise distance ($\dN{2}$) between proteins with different similarity group: sequence similarity within $[0,1/3)$, $[1/3, 2/3)$ and $[2/3, 1]$; see Appendix for details related to sequence similarity calculation. Each box is sampled, with $N=5000$, from all human-mouse protein pairs. All paired differences are statistically significant based on the $t$-test after Bonferroni correction. }
    \label{fig:semantic_dist_ontologies}
\end{figure}

\begin{figure} [t]
    \centering
    \includegraphics[width=.7\columnwidth]{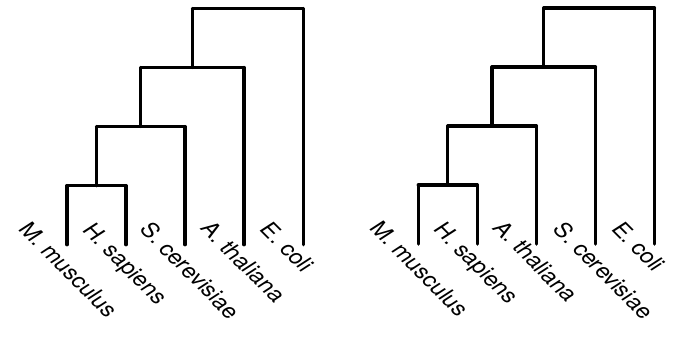}
    \caption{Functional phylogenetic trees for \textit{H.~sapiens}, \textit{M.~musculus}, \textit{S.~cerevisiae}, \textit{A.~thaliana} and \textit{E.~coli} in the Molecular function and Cellular component ontologies (left, \emph{correct}) and the Biological process ontology (right, \emph{incorrect}). See Appendix for further details.}
    \label{fig:phylogenetic_tree}
\end{figure}

\section{Related Work}
\label{relatedwork}

\subsection{Metric learning}
Learning distance metrics has recently emerged as one of the important topics in  machine learning and data mining \cite{xing2003distance}. In this approach, a metric itself depends on a set of parameters which are then learned with respect to a specific problem, data set and algorithm at hand; e.g., KNN \cite{weinberger2009distance}, K-means \cite{bilenko2004integrating}, sometimes under constraints such as sparseness. The most studied metric in this field is the Mahalanobis metric $d(\mathbf{x},\mathbf{y}) = \sqrt{(\mathbf{x}-\mathbf{y})^T\mathbf{A}(\mathbf{x}-\mathbf{y})}$, where $\mathbf{x},\mathbf{y} \in \mathbb{R}^k$ and $\mathbf{A} \in \mathbb{R}^{k\times k}$, which usually leads to convex formulations. This approach has important merits such as application-specific optimality. However, there is little theoretical understanding related to the consistency of metric learning 
\cite{bellet2013survey} as well as scalability issues caused by either large parameter space or the data set size \cite{yang2006distance,weinberger2009distance}. In contrast, default metrics, though not tailored for specific problems, are immediately available and usually present the first line of attack to a specific learning task. Incorporating the $d^p$ metrics into the metric learning framework might be an interesting research direction. 


\subsection{Default metrics} A body of research exists on handcrafting domain-specific distance functions \cite{book}. Distances on strings \cite{Yujian2007, Li2004}, rankings \cite{Kumar2010, Hassanzadeh2014}, or graphs \cite{Cao2013} have been actively researched in information retrieval, computational biology, computer vision, etc. Similarly, metrics on probability distributions have long been theoretically studied \cite{zolotarev1983probability}. Different metrics emerge for different reasons: some originated in functional analysis, such as the $L^p$ metric and the uniform metric, yet others due to their special properties; e.g., the Hellinger distance which admits decomposition under certain conditions \cite{zolotarev1983probability}. One application for such metrics is in stochastic programming and stability analysis in related problems \cite{rachev2002quantitative}. They are also used in statistical inference \cite{rao1973linear} or applied to measure the within- and between-population diversity in economics, genetics, etc. \cite{rao1982diversity}.

\section{Conclusions}
\label{conclusion}

This work was motivated by the desire to develop a family of metrics for learning across different domains, especially on high-dimensional and structured data that characterize many modern applications. Overall, we believe that the class of functions proposed in this work present sensible choices in various fields and believe that their good theoretical properties and strong empirical performance will play a positive role in their adoption. 





\bibliographystyle{abbrv}






\newpage

\begin{center}
\textbf{\Large{Appendix}}
\end{center}

\section{Lemmas and Proofs}
\label{lemma_proof}
%
\begin{proof}[Proof of Theorem \ref{ThmUnnormalizedSet}]
The only property of a metric not obviously satisfied by $d$ is the triangle inequality. Given arbitrary sets $A, B, C \in X$, we have
\begin{align*}
& \phantom{=.}d(A,B) + d(B,C) \\ & = (|A-B|^p+|B-A|^p)^\frac{1}{p}+(|B-C|^p+|C-B|^p)^\frac{1}{p}\\
                    &\geq ((|A-B|+|B-C|)^p+(|B-A|+|C-B|)^p)^\frac{1}{p}  \\
                    &\geq (|A-C|^p +|C-A|^p)^\frac{1}{p} \\ 
              	   & =d(A,C). \\
\end{align*}
The first inequality holds due to Minkowski inequality. 
The second inequality can be deduced from triangle inequality of Manhattan distance on binary vectors.
\end{proof}
\begin{proof}[Proof of Theorem \ref{ThmSet}]
As in the unnormalized case we only prove the triangle inequality. Let $A, B, C \in X$ be arbitrary sets. If at least one of $|A\cup B|=0$ and $|B\cup C| = 0$ holds, then the triangle inequality holds. So we only consider the cases when $|A\cup B|\neq 0$ and $|B\cup C| \neq 0$.
%
\noindent Let the cardinality of different sets shown in Figure \ref{fig:venn} be denoted by $a$, $b$, $\ldots$, $g$. Let $\tau = a + b + \dots +g $ and $\hp(x,y) = (x^p+y^p)^{\nicefrac{1}{p}}$. 
\begin{align*}
 & \phantom{=.}d_N(A,B) + d_N(B,C)  \\
 & = \frac{\hp(a+d,b+f)}{\tau-c}+\frac{\hp(b+e,c+d)}{\tau-a}\\
                    &\geq\frac{\hp(a+d,b+f)}{\tau}+\frac{\hp(b+e,c+d)}{\tau}\\
                    &\geq \frac{\hp(a+d+b+e,b+f+c+d)}{\tau}\\
                    &=   \hp\left(\frac{a+d+b+e}{\tau},\frac{b+f+c+d}{\tau}\right) \\
                     &\geq   \hp\left(\frac{a+d+b+e-b}{\tau-b},\frac{b+f+c+d-b}{\tau-b}\right) \\
                    &\geq\frac{\hp(a+d+e,f+c+d)}{\tau-b}\\
                    &\geq\frac{\hp(a+e,f+c)}{\tau-b}\\
              	   &=d_N(A,C).\\
\end{align*}
\begin{figure}[ht]
\centering
\includegraphics[scale=.9]{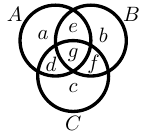}
\protect\caption{The Venn diagram and cardinality related to sets $A$, $B$ and $C$. \label{fig:venn}}
\end{figure}
The second inequality holds due to Minkowski inequality. The third inequality is true since we subtracted the same nonnegative number $b$ from both the numerator and denominator of the fraction with the fraction itself remaining in $[0,1]$ after the subtraction (the numerator is nonnegative before and after the subtraction). Hence, $d_N$ is a metric. 
It follows that $d_N$ is bounded in $[0,1]$ via the Minkowski inequality.
\end{proof}
%
\begin{proof}[Proof of Theorem \ref{ThmVectorU} and Theorem \ref{ThmVectorN}]
The metrics $d$ and $d_N$ defined in Eqs.~\ref{eq:5}-\ref{eq:6} can be reduced to be applied on vectors in $\R^k$ by rewriting $\mathbf{x} =(x_1,x_2,\ldots,x_k)$ as a map $f(t):=\sum x_i\mathbf{1}_{i-1 <t\leq i}$, where $\mathbf{1}_{(\cdot)}$ is an indicator function. Since $d$ and $d_N$ are metrics, proved in Theorem \ref{Thm1} and Theorem \ref{ThmR}, we obtain that $d$ and $d_N$ on $\R^k$ are also metrics.
\end{proof}
\begin{lem}
For any two real functions $f$ and $g$, we have $(f+g)^{+}\leq f^{+}+g^{+}$, and, $(f+g)^{-}\leq f^{-}+g^{-}$.
\label{Lemma1}
\end{lem}
\begin{proof}
We will prove $(f+g)^{+}\leq f^{+}+g^{+}$ first.
 Since
\[f^{+}+g^{+} = \max(f,0)+\max(g,0) \geq f+g,\]
and
\[f^{+}+g^{+} = \max(f,0)+\max(g,0) \geq 0,\]
we have that $ f^{+}+g^{+} \geq \max(f+g,0)=(f+g)^+$.

Notice that $f^{-} = (-f)^{+}$. To prove $(f+g)^{-}\leq f^{-}+g^{-}$, we have $(f+g)^{-} = (-f-g)^{+} \leq(-f)^{+}+(-g)^{+}=f^{-} + g^{-}$.
\end{proof}

\begin{lem}
For any real functions $f$, $g$ we have
\[
f^+ + g^+ = \min(|f|,|g|)\boldsymbol{1}_{\{ fg <0\}} +(f+g)^+.\]
\label{Lemma2}
\end{lem}
\begin{proof}
For $fg \geq 0$, it is not hard to see that $ f^+ + g^+ = (f+g)^+$.
When $ fg < 0$, without loss of generality, suppose that $f < 0$ and $g >0$. Then we have
\begin{align*}
 & \textstyle f^+ + g^+ -(f+g)^+  = g - (f+g)^+ \\
 & = \textstyle\begin{cases} g, & \text{if 
  \,} |f| > |g| \\ -f, & \text{if \,}|f| \leq |g| \end{cases} 
  = \min(|f|,|g|). 
\end{align*} 
\end{proof}

\begin{lem}
For any real functions $f,g$ and $h$, we have
\begin{align*}
       & \textstyle(f - g)^+ + (g - h)^+
   - \{\max(|g|,|f-g|,|g-h|)- \max(|f|,|h|,|f-h|)\}^+ 
  \geq  (f - h)^+.
\end{align*}
\label{TheInequality}
\end{lem}
\begin{proof}
By Lemma \ref{Lemma2}, it is equivalent to show  
\begin{align*}
      \min(|f - g|, |g - h|) \boldsymbol{1}_{\{(f-g)(g-h)<0\}} 
 \geq \{\max(|g|,|f-g|,|g-h|) - \max(|f|,|h|,|f-h|)\}^+.
\end{align*}
Since $ \min(|f - g|, |g - h|) \boldsymbol{1}_{\{(f-g)(g-h)<0\}} \geq 0 $, it suffices to show that it is no less than $\max(|g|,|f-g|,|g-h|) - \max(|f|,|h|,|f-h|)$.

Consider the case $(f-g)(g-h) \geq 0$ first. We have $f \geq g \geq h$ or $f \leq g \leq h$. $g$ is located between $f$ and $h$, therefore $|g| \leq \max{(|f|, |h|)}$ and $\max(|f-g|,|g-h|)\leq |f-h|$.
Then we have 
\begin{align*}
|g|-\max(|f|,|h|,|f-h|) \leq |g| - \max(|f|,|h|) \leq 0,
\end{align*}
and
\begin{align*}
\max(|f-g|,|g-h|)-\max(|f|,|h|,|f-h|) 
\leq  \max(|f-g|,|g-h|)-|f-h| \leq 0.
\end{align*}
This shows that 
\begin{align*}
 \max(|g|.|f-g|,|g-h|)-\max(|f|,|h|,|f-h|)
 \leq  \; 0 =  \min(|f - g|, |g - h|) \boldsymbol{1}_{\{(f-g)(g-h)<0\}}
\end{align*}
When $(f-g)(g-h) \leq 0$, we have the following two cases 
\begin{enumerate}
\item $\min(|f - g|, |g - h|)=|f-g|$
\item $\min(|f - g|, |g - h|)=|g-h|$
\end{enumerate}
For case 1, we want to show that $ |f - g|\geq \max(|g|,|g-h|) - \max(|f|,|h|,|f-h|)$. This is true since
\[
|f-g| \geq |g| - |f| \geq |g| -  \max(|f|,|h|,|f-h|),
\]
and
 \[
|f-g| \geq |g-h| - |f-h|\geq |g-h| -  \max(|f|,|h|,|f-h|).
\]
Combining those two inequalities we get $ |f - g|\geq \max(|g|,|g-h|) - \max(|f|,|h|,|f-h|)$.
For case 2, we want to show  $ |g - h|\geq \max(|g|,|f-g|) - \max(|f|,|h|,|f-h|)$. By the same analogy, since
\[
|g-h| \geq |g| - |h| \geq |g| -  \max(|f|,|h|,|f-h|),
\]
and
\[
|g-h| \geq |f-g| - |f-h|\geq |f-g| -  \max(|f|,|h|,|f-h|),
\]
we have $ |g - h|\geq \max(|g|,|g-h|) - \max(|f|,|h|,|f-h|)$.
Therefore this inequality still holds for $(f-g)(g-h) \leq 0$.
\end{proof}
\begin{lem}
\label{Trivial}
For any two real functions $f$ and $g$, we have
\[
   \max(f,g) = f + (g - f)^+.
\]
\end{lem}
%
%
\begin{lem}
    Any Cauchy sequence $\{f_n\}$ in $(L(\R),d_N)$  is bounded in $L^1$; i.e., $\int|f_n|\,dx < M$ for some constant $M>0$.
\label{bounded}    
\end{lem}
\begin{proof}
We instead prove the contrapositive version of the above statement. Suppose $\{f_n\}$ is a sequence in $(L(\R),d_N)$ that is unbounded in $L^1$, or equivalently $\int|f_n|\,dx = \infty$ as $n\rightarrow\infty$. Thus, we have 
\begin{align*}
&d_N(f_n,f_m) \\
=& \dfrac{((\int(f_n-f_m)^-\,dx)^p+(\int(f_n-f_m)^+\,dx)^p)^\frac{1}{p}}{\int \max(|f_n|,|f_m|,|f_m-f_n|)\,dx}\\
 \geq &\dfrac{((\int(f_n-f_m)^-\,dx)^p+(\int(f_n-f_m)^+\,dx)^p)^\frac{1}{p}}{\int|f_m|+|f_n|)\,dx}
\end{align*}
For any integer $N_0>0$, pick $n$ to be $N_0$ and we have that $d_N(f_{N_0},f_m) $ is given by
\begin{align*}
& \dfrac{((\int(f_{N_0}-f_m)^-\,dx)^p+(\int(f_{N_0}-f_m)^+\,dx)^p)^\frac{1}{p}}{\int|f_m|+|f_{N_0}|)\,dx}\\
 = &((\int(A_m-B_m)^-\,dx)^p
 +(\int(A_m-B_m)^+\,dx)^p)^\frac{1}{p}
\end{align*}
where 
\[
A_m = \frac{f_{N_0}}{\int|f_m|+|f_{N_0}|)\,dx}\textrm{,}\qquad B_m = \frac{f_m}{\int|f_m|+|f_{N_0}|)\,dx}
\]
and $\{A_m\}$ and $\{B_m\}$ are sequences of functions. Fixing $N_0$ and sending $n$ to $\infty$ we have that $\int|A_m|\,dx\rightarrow0$ and  $\int|B_m|\,dx\rightarrow1$ as $m\rightarrow\infty$. Therefore, we could choose $m>N_0$ such that $\int|A_m|\,dx<\nicefrac{1}{10}$ and  $\int|B_m|\,dx >\nicefrac{9}{10}$. Then between $\int B_m^-\,dx$ and $\int B_m^+\,dx$ there is at least one greater than $\nicefrac{9}{20}$. Without loss of generality, suppose $\int B_m^-\,dx>\nicefrac{9}{20}$, then
\begin{align*}
d_N(f_{N_0},f_m) 
& \geq (\textstyle\int(A_m-B_m)^+\,dx)^p)^\frac{1}{p}\\
& = \textstyle\int(B_m-A_m)^-\,dx\\
&\geq \textstyle\int(B_m)^-\,dx-\int(A_m)^-\,dx\\
&\geq \textstyle\int(B_m)^-\,dx-\int|A_m|\,dx \\
& >\frac{7}{20}\\
\end{align*}
\noindent With this we have shown that for any $N_0>0$, there exist $m,n\geq N_0$ such that $d_N(f_n,f_m)>\nicefrac{7}{20}$; i.e., $\{f_n\}$ is not a Cauchy sequence in $(L(\R),d_N)$. Thus, we have proved the claim.
\end{proof}
\begin{proof}[Proof of Theorem \ref{ThmOntologyU}]
To show that $d$ is a metric is analogous to the proof of Theorem \ref{ThmUnnormalizedSet}. Let $A, B, C$ be arbitrary consistent subgraphs of the ontology. Instead of $|A-B|$, we use $\sum_{v \in A-B} ia(v)$ and similarly for the other cardinalities. The proof follows line for line after these substitutions to the proof of Theorem \ref{ThmUnnormalizedSet}.

To prove $d_N$ in Theorem \ref{ThmOntologyN} is a metric, we follow a similar argument to the proof of Theorem \ref{ThmSet}. The analogues of $a,b,c,d,e,f,g$ here are
\begin{align*}
a & = \sum_{v \in A-(B \cup C)}ia(v),
& b & = \sum_{v \in B-(A \cup C)}ia(v), 
& c & =\sum_{v \in C-(A \cup B)}ia(v),
& d & = \sum_{v \in A\cap C - B}ia(v), \\
 e & = \sum_{v \in A\cap B - C}ia(v),
& f & = \sum_{v \in B\cap C - A}ia(v), 
& g & = \sum_{v \in A\cap B \cap C}ia(v). 
\end{align*}
With these substitutions, the proof is exactly the same as that of Theorem \ref{ThmSet}.



Invoking the Minkowski inequality, we obtain that
\begin{align*}
d_{N}(F,G) & \leq \frac{ru(F,G)+mi(F,G)}{\sum_{v\in F\cup G}ia(v)}
            \leq \frac{\sum_{v\in F\cup G}ia(v)}{\sum_{v\in F\cup G}ia(v)} =1\,.
\end{align*}
Since $d_N$ is nonnegative, we obtain that $d_N \in [0,1]$.
\end{proof}
%
%
\begin{proof}[Proof of Theorem \ref{Thm1}]
Since $d_N$ is non-negative, the equations $d_N(f,g)=d_N(g,f)$ and $d_N(f,g)=0$ hold if and only if $f=g$ almost everywhere, it suffices to show that $d$ satisfies the triangle inequality. Let $f$, $g$, and $h$ be in $L(\R)$. Then we have
%
\small
\begin{align*}
& \phantom{=.} D(f,g)+D(g,h)\\   
&=\left(\left(\int(f(x)-g(x))^{+}dx\right)^p + \left(\int(f(x)-g(x))^{-}dx\right)^p\right)^{\frac{1}{p}}\\
              &+\left(\left(\int(g(x)-h(x))^{+}dx\right)^p + \left(\int(g(x)-h(x))^{-}dx\right)^p\right)^{\frac{1}{p}}\\
              &\geq \left(\left(\int(f(x)-g(x))^{+}+(g(x)-h(x))^{+}dx\right)^p + \left(\int(f(x)-g(x))^{-}+(g(x)-h(x))^{-}dx\right)^p\right)^{\frac{1}{p}}\\
              &\geq \left(\left(\int(f(x)-g(x)+g(x)-h(x))^{+}dx\right)^p +\left(\int(f(x)-g(x)+g(x)-h(x))^{-}dx\right)^p\right)^{\frac{1}{p}}\\
              &\geq \left(\left(\int(f(x)-h(x))^{+}dx\right)^p +\left(\int(f(x)-h(x))^{-}dx\right)^p\right)^{\frac{1}{p}}\\
              & =D(f,h).
\end{align*}
\normalsize
Therefore, $d$ is a metric.
\end{proof}
\begin{proof}[Proof of Theorem \ref{ThmR}]
It is easy to check that $d_N$ is non-negative, $d_N(f,g)=d_N(g,f)$ and $d_N(f,g)=0$ if and only if $f=g$ almost everywhere. Therefore, it remains to be shown that the inequality $d_N(f,g)+d_N(g,h)\geq d_N(f,h)$ is satisfied. 

Let $f$, $g$, and $h$ be bounded functions in $L(\R)$. To begin, let us look at the trivial cases. Define $\M(f,g) = \int \max(|f|,|g|,|f-g|)\,dx $ and $\M^*(f,g,h) = \int\max(|f|,|g|,|h|,|f-g|,|g-h|,|f-h|)\,dx$.

\begin{enumerate}
  \item [a.]If $\int \max(|f|,|g|,|f-g|)\,dx=0$, then $f=g$ almost everywhere. Consequently $d_N(f,h)=d_N(g,h)$ and $d_N(f,g)=0$, so the inequality holds.
  \item [b.]If $\int\max(|f|,|h|,|f-h|)\,dx=0$, then $d_N(f,h)=0$, in which case the inequality is true due to the non-negativity of $d_N$.
  \item [c.]If $\int\max(|g|,|h|,|g-h|)\,dx=0$, then $g=h$ almost everywhere and $d_N(f,g)=d_N(f,h)$; thus, the triangle inequality still holds.
\end{enumerate}

\noindent Next we consider the case where none of the three denominators is zero. 
{\allowdisplaybreaks\begin{align*}\allowdisplaybreaks[4]
 &\phantom{=.} D_N(f,g)+D_N(g,h) \\
 &=\frac{\left(\left(\int(f-g)^{+}\,dx\right)^p + \left(\int(f-g)^{-}\,dx\right)^p\right)^{\frac{1}{p}}}{\M(f,g)}
             +\frac{\left(\left(\int(g-h)^{+}\,dx\right)^p + \left(\int(g-h)^{-}\,dx\right)^p\right)^{\frac{1}{p}}}{\M(g,h)}\\
              & \geq \left(\left(\frac{\int(f-g)^{+}\,dx}{\M(f,g)}+\frac{\int(g-h)^{+}\,dx}{\M(g,h)}\right)^p  + \left(\frac{\int(f-g)^{-}\,dx}{\M(f,g)}+\frac{\int(g-h)^{-}\,dx}{\M(g,h)}\right)^p\right)^{\frac{1}{p}}\\
              &\geq \left(\left(\frac{\int(f-g)^{+}+(g-h)^{+}\,dx}{\M^*(f,g,h) }\right)^p + \left(\frac{\int(f-g)^{-}+(g-h)^{-}\,dx}{\M^*(f,g,h) }\right)^p\right)^{\frac{1}{p}}\\
&=(I^p+J^p)^\frac{1}{p},
\end{align*}}
\noindent where 
\[I = \textstyle\frac{\int(f-g)^{+}+(g-h)^{+}\,dx}{\M^*(f,g,h) } \hspace{3pt}  \quad \text{and} \quad \hspace{3pt}  J = \frac{\int(f-g)^{-}+(g-h)^{-}\,dx}{\M^*(f,g,h) }.\]
 Let $\Gamma(f,g,h) = \int (\max(|g|,|f-g|,|g-h|)-\max(|f|,|h|,|f-h|))^{+}\,dx $. By subtracting $\Gamma(f,g,h)$ from the numerator and denominator of $I$ at the same time, it follows that
\begin{align*}
 I
& \geq  \frac{\int(f-g)^{+}+(g-h)^{+}\,dx-\Gamma(f,g,h)}{ \M^*(f,g,h)  \,-\Gamma(f,g,h)}\\
& = \frac{\int(f-g)^{+}+(g-h)^{+}\,dx-\Gamma(f,g,h)}{\M(f,h)}\\
& \geq \frac{\int(f-h)^{+}\,dx}{\M(f,h)}.
\end{align*}

\noindent The first of the above inequalities holds since we are subtracting a non-negative number no greater than the non-negative numerator from the top and bottom while the fraction stays in $[0,1]$. The equality holds due to Lemma \ref{Trivial} and the last inequality due to Lemma \ref{TheInequality}.
By analogy it can be shown that
\begin{align*}
 J\geq \frac{\int(f-h)^{-}\,dx}{\M(f,h)}.
\end{align*}
\noindent Thus, we have $d_N(f,g)+d_N(g,h)\geq d_N(f,h)$.

For general functions $f,g,h \in L(\R)$, we can exclude the set $(|f|=\infty)\cup(|g|=\infty)\cup(|h|=\infty)$, since the set where the functions are infinite is of measure zero. Thus, the theorem would proceed the same way since Lemma \ref{Lemma2} and Lemma \ref{TheInequality} still hold for $|f|,|g|,|h| <\infty$.

Now we show that $d_N \in [0,1]$. Since  $d_N(\cdot ,\cdot)$ is nonnegative, we only need to show that it is bounded by 1. 
Applying the Minkowski inequality we have that
\begin{align*}
  & \left(\displaystyle(\int(f-g)^{+}\,dx)^p+(\int(f-g)^{-}\,dx)^p\right)^\frac{1}{p} \\
  \leq & \int(f-g)^{+}\,dx +  \int(f-g)^{-}\,dx \\
  & =  \int|f-g|\,dx.
\end{align*}

Thus the numerator is bounded by the denominator and the fraction is no greater than 1.
With $\frac{0}{0}:= 0$ for $d_N(\cdot ,\cdot)$, we have shown that this metric is in $[0,1]$.
\end{proof}
\begin{proof}[Proof of Proposition \ref{equivalence}]
Since the proposition holds for $p=1$ and $p=\infty$, we only consider cases when $p>1$. Consider $\mathbf{x},\mathbf{y}\in\R^k$. First we show that $ d_{\textrm{M}}(\mathbf{x},\mathbf{y}) \leq d(\mathbf{x},\mathbf{y}) $. We have that
\begin{align*}
d_{\textrm{M}}(\mathbf{x},\mathbf{y})^p & = \sum|x_i - y_i|^p \\
& = \sum_{i:x_i \geq y_i} |x_i - y_i|^p+  \sum_{i:x_i<y_i} |x_i - y_i|^p \\
&\leq (\sum_{i:x_i \geq y_i} x_i - y_i)^p+  (\sum_{i:x_i<y_i} y_i - x_i)^p\\
& = d(\mathbf{x},\mathbf{y})^p
\end{align*}
The inequality holds since the bases of the exponents are positive. Now we show the other inequality holds, that is 
\begin{align*}
    d(\mathbf{x},\mathbf{y})^p &= (\sum_{i:x_i \geq y_i} x_i - y_i)^p+  (\sum_{i:x_i<y_i} y_i - x_i)^p\\
                               & \leq \{[k^{1/P^*}(\sum_{i:x_i \geq y_i}(x_i - y_i)^p)^{1/p}]^p  + [k^{1/P^*}(\sum_{i:x_i<y_i}(y_i - x_i)^p)^{1/p}]^p\} \\
                               & = k^{p/p^*}  \sum|x_i - y_i|^p\\
                               & = k^{p/p^*} d_{\textrm{M}}(\mathbf{x},\mathbf{y})^p,\\
\end{align*}
\noindent with $p^* = p/(p-1)$.
The inequality holds since $|\sum_{i=1}^m a_i| = |(1,\ldots,1)\cdot\mathbf{a}|\leq \|(1,\ldots,1)\|_{p^*}\|\mathbf{a}\|_p$, where $\|\cdot\|_p$ represents the $L^p$ norm, thanks to H\"{o}lder's inequality.
\end{proof}
\begin{proof}[Proof of Theorem \ref{completeness}]
By definition, a metric space $(X,d)$ is complete if all Cauchy sequences in $X$ converge in $X$; that is, if the limit point of every Cauchy sequence in $X$ remains in $X$. 
Let us first consider a Cauchy sequence in $(L(\R),d)$, where for a given $\epsilon >0$, there exists some $N>0$ such that $d(f_n,f_m)<\epsilon$ for all $n,m \geq N$; i.e., $(\int(f_n-f_m)^-\,dx)^p+(\int(f_n-f_m)^+\,dx)^p < \epsilon^p$. It follows that $\int(f_n-f_m)^-\,dx<\epsilon$ and $\int(f_n-f_m)^+\,dx<\epsilon$ and thus $\int|f_n-f_m|\,dx<2\epsilon$. Therefore, $\{f_n\}$ is a Cauchy sequence in $L^1$ space, where the metric in $L^1$ is $d(f,g) = \int|f-g|\,dx$ for integrable functions and thus $f_n$ converges to a function $f$ in $L^1$ by the completeness of $L^1$ space.

Now we look at a Cauchy sequence $\{f_n\}$ in $(L(\R),d_N)$.
By Lemma \ref{bounded} we have that $\int|f_n|\,dx \leq M$ for all $n$ for some positive constant $M$. It follows that for any given $\epsilon >0$, there exists some integer $N_0>0$ such that $d_N(f_n,f_m)<\epsilon$ for all $n,m \geq N_0$, or in other words, $d(f_n,f_m)<2M\epsilon$. Therefore, $\{f_n\}$ is Cauchy in $(L(\R),d)$ and by previous results we know that $\{f_n\}$ has a limit in $L^1$ and therefore $(L(\R),d_N)$ is complete.
\end{proof}

\section{Real-valued and text data} \label{dataset}

Real-valued data sets were downloaded from UCI machine learning repository. The three numbers in the parenthesis after each data set listed below correspond to (number of classes, number of instances, number of features): \textit{airfoil}~(2, 1503, 5), \textit{banknote}~(2, 1372, 4), \textit{cardiotocography}~(10, 2126, 21), \textit{concrete}~(2, 1030, 8), \textit{eyestate}~(2, 14980, 14), \textit{faults}~(7, 1941, 27), \textit{fertility}~(2, 100, 9), \textit{gas}~(6, 13910, 128), \textit{glass}~(6, 214, 10), \textit{housing}~(2, 506, 13), \textit{ionosphere}~(2, 351, 34), \textit{iris}~(3, 150, 4), \textit{landsat}~(6, 6435, 36), \textit{leaf}~(30, 340, 14), \textit{pageblock}~(5, 5473, 10), \textit{pendigits}~(10, 10992, 16), \textit{pima}~(2, 768, 8), \textit{retinopathy}~(2, 1151, 19), \textit{seeds}~(3, 210, 7), \textit{segment}~(7, 2310, 19), \textit{shuttle}~(7, 58000, 9), \textit{sonar}~(2, 208, 60), \textit{spambase}~(2, 4601, 57), \textit{transfusion}~(2, 748, 4), \textit{vertebral}~(2, 310, 6), \textit{waveform}~(3, 5000, 40), \textit{wdbc}~(2, 569, 30), \textit{wilt}~(2, 4839, 5), \textit{winequality}~(7, 6497, 11), \textit{yeast}~(10, 1484, 8). \textit{Concrete} and \textit{housing} were converted to binary classification tasks based on the target mean.


Among the five text document data sets, \textit{lifesci5}~(5, 9000, 52757) was extracted from a collection of abstracts from 5 life sciences journals: \textit{Scientific Reports}, \textit{Oncotarget}, \textit{Proceedings of the National Academy of Sciences of the United States of America}, \textit{ACS Applied Materials and Interfaces} and \textit{PloS One}, obtained from the Europe PMC life science database. As for \textit{reuters}~(4, 2065, 8943), documents were collected by selecting these 4 topics exclusively: ``interest'', ``trade'', ``grain'' and ``crude''. Each of these data sets was pre-processed by stop-word removal and then Porter stemming. The remaining pre-processed data sets were downsampled \textit{Cade12}~(12, 4800, 57312), \textit{webkb}~(4, 2803, 7288) downloaded from \cite{Cachopo2007}, \textit{20newsgroups}~(20, 4000, 130107) from scikit-learn python library, \textit{moviereview}~(2, 2000, 39659) from \cite{Pang2004}, \textit{CNAE9}~(9, 1080, 856), \textit{TTC3600}~(6, 3600, 5692) from the UCI Machine Learning Repository, \textit{sdm06}~(27, 930, 99899) from \cite{Dalkilic2006} and \textit{bbc}~(5, 2225, 9635) from \cite{greene06icml}. All text document data sets used tf-idf as features.

\section{Protein function data}
\label{protein_data}
Protein function data were downloaded from the Swiss-Prot database (July 2015). 
In particular, we collected protein functions for the following model organisms where sufficient annotations are available: \textit{Homo sapiens}, \textit{Mus musculus}, \textit{Arabidopsis thaliana}, \textit{Saccharomyces cerevisiae}, and \textit{Escherichia coli}. Only those annotations with (experimental) evidence codes EXP, IDA, IMP, IPI, IGI, IEP, TAS, and IC were considered. Table~1 summarizes the data sets: here the genome size corresponds to the total number of proteins available for each species in Swiss-Prot. The following three columns show the number of proteins that are annotated in the three domains of Gene Ontology accordingly.

%

\begin{table}[ht]
\centering
\small
\begin{tabular}{|c|r|r|r|r|}
\hline 
Organism & Genome size & MFO & BPO & CCO\tabularnewline
\hline 
\hline 
\textit{H. sapiens} & 20,193 & 11,979 & 11,398 & 12,691\tabularnewline
\hline 
\textit{M. musculus} & 16,733 &  6,728 &  7,702 &  7,322\tabularnewline
\hline 
\textit{A. thaliana} & 14,305 &  4,266 &  5,749 &  5,950 \tabularnewline
\hline 
\textit{S. cerevisiae} & 6,720 & 4,051 &  4,676 &  4,102\tabularnewline
\hline 
\textit{E. coli} & 4,433 & 2,272 &  2,331 &  2,119\tabularnewline
\hline 
\end{tabular}
\label{ontology_data}
\caption{Data set sizes for the five organisms used in this work. The genome size refers to the number of protein sequences available in Swiss-Prot for each species.}
\end{table}

The conditional probability tables were estimated using the maximum likelihood approach from the entire set of functionally annotated proteins in Swiss-Prot. This set included $72977$ proteins from $1576$ species with MFO terms, $92874$ proteins from $1503$ species with BPO terms, and $89693$ proteins from  $862$ species with CCO terms.

\section{Calculation of protein sequence similarity}
\label{protein_seq}
The sequence similarity of two protein sequences was measured as the ratio of the number of identical characters in the alignment and the length of the longer protein sequence. Sequence alignment was obtained using a Needleman-Wunsch algorithm with BLOSUM62 similarity matrix, gap opening penalty of 11 and gap extension penalty of 1.

\section{Phylogenetic clustering}
\label{phylogeny}
The functional phylogenetic trees with respect to a group of organisms are generated using single-linkage hierarchical clustering \cite{Tan2006}. This algorithm starts by considering every data point (species) to be a cluster of one element and in each step merges the two closest clusters. The algorithm continues until all original data points belong to the same cluster. The distance between species is based on pairwise distances between functionally annotated proteins as described below. For simplicity, we use normalized semantic distance from Eq.~\ref{eq:4} with $p=1$ in all experiments.

Without loss of generality, we illustrate the species distance calculation by showing how to compute the distances between \emph{A.~thaliana} (A) and all other organisms using protein function data only. An important challenge in this task arises from unequal genome sizes as well as unequal fractions of experimentally annotated proteins in each species (Table~1), making most distance calculation techniques unsuitable for this task. We therefore carry out sampling to compare species using a fixed yet sufficiently large set of $N$ proteins from each species. The algorithm first samples (with replacement) $N=1000$ proteins from each species. It then counts the number of times the proteins from \textit{E.~coli} (E), \textit{H.~sapiens} (H), \textit{M.~musculus} (M) and \textit{S.~cerevisiae} (Y) are functionally most similar to proteins in \emph{A.~thaliana}, with ties resolved uniformly randomly. These counts are used to calculate the directional distances between \emph{A.~thaliana} and the remaining four species. The procedure is repeated $B=1000$ times with different bootstrap samples to stabilize the results. The details of the algorithm are shown in Algorithm 1.

\section{Additional results}
We carried out several additional experiments in order to evaluate the proposed distance functions. These experiments investigated the influence of data normalization (z-score, min-max, unit) and performance assessment criteria. These results are provided in a table at

\href{www.cs.indiana.edu/~predrag/table.xlsx}{\texttt{www.cs.indiana.edu/{\raise.17ex\hbox{$\bm{\scriptstyle\sim}$}}predrag/table.xlsx}}.

\noindent Although the effect of data normalization deserves attention, all conclusions reached in the main portion of the paper remain unchanged.

\begin{figure}[ht!]
\label{alg:phylogeny}
\begin{center}
\includegraphics[width=0.75\linewidth,trim=4cm 13.5cm 4cm 3.5cm,clip
]{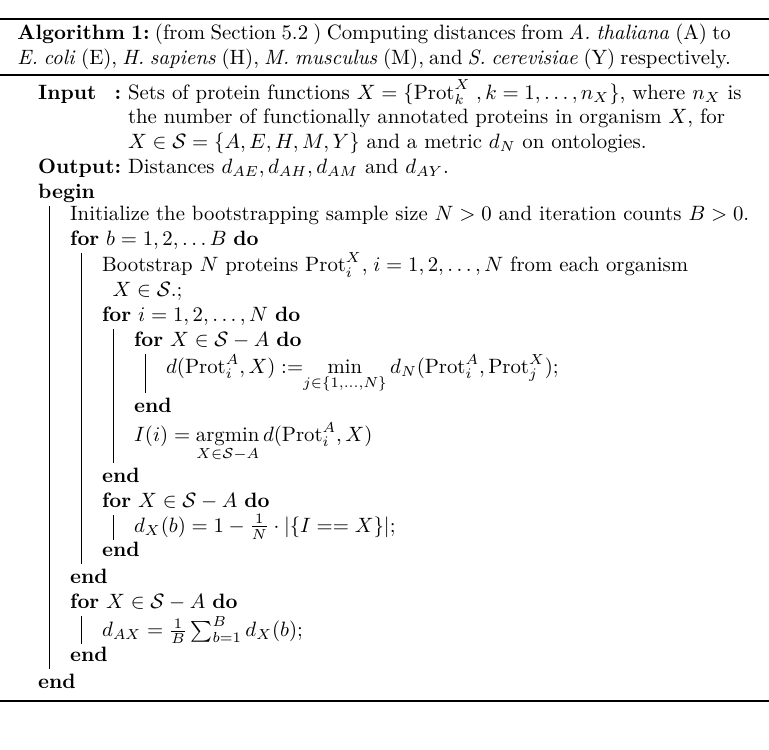}
\end{center}
\end{figure}

\end{document}